\documentclass{ws-ijait}
\usepackage[super]{cite}
\usepackage{url}
\usepackage{diagbox}
\usepackage{makecell}
\usepackage{amsmath,amssymb,amsfonts}
\usepackage{graphicx}
\usepackage{hyperref}
\hypersetup{colorlinks,linkcolor=blue,citecolor=blue,urlcolor=magenta, linktocpage,plainpages=false, bookmarksopen = true}

\usepackage[caption=false]{subfig}
\usepackage{romannum}
\usepackage{soul}

\begin{document}

\markboth{Shuyue Guan, Murray Loew}
{A Distance-based Separability Measure for Internal Cluster Validation}

%
\catchline{}{}{}{}{}
%

\title{A Distance-based Separability Measure for Internal Cluster Validation}

\author{Shuyue Guan\footnote{ORCID: \url{https://orcid.org/0000-0002-3779-9368}}}

\address{Department of Biomedical Engineering, The George Washington University,\\
800 22nd St NW, Washington, DC 20052, USA\\
frankshuyueguan@gwu.edu}

\author{Murray Loew\footnote{Corresponding author}}

\address{Department of Biomedical Engineering, The George Washington University,\\
800 22nd St NW, Washington, DC 20052, USA\\
loew@gwu.edu}

\maketitle


\begin{abstract}
To evaluate clustering results is a significant part of cluster analysis. Since there are no true class labels for clustering in typical unsupervised learning, many internal cluster validity indices (CVIs), which use predicted labels and data, have been created. Without true labels, to design an effective CVI is as difficult as to create a clustering method. And it is crucial to have more CVIs because there are no universal CVIs that can be used to measure all datasets and no specific methods of selecting a proper CVI for clusters without true labels. Therefore, to apply a variety of CVIs to evaluate clustering results is necessary. In this paper, we propose a novel internal CVI -- the Distance-based Separability Index (DSI), based on a data separability measure. We compared the DSI with eight internal CVIs including studies from early Dunn (1974) to most recent CVDD (2019) and an external CVI as ground truth, by using clustering results of five clustering algorithms on 12 real and 97 synthetic datasets. Results show DSI is an effective, unique, and competitive CVI to other compared CVIs. We also summarized the general process to evaluate CVIs and created the rank-difference metric for comparison of CVIs’ results. \textit{The code can be found in author’s website linked up with the ORCID: \url{https://orcid.org/0000-0002-3779-9368}.}
\end{abstract}


\section{Introduction}
Cluster analysis is an important unsupervised learning method in machine learning. The clustering algorithms divide a dataset into clusters \cite{Jain1999Data} based on the distribution structure of the data, without any prior knowledge.
Clustering is widely studied and used in many fields, such as data mining, pattern recognition, object detection, image segmentation, bioinformatics, and data compression \cite{Roiger2017Data,Wen2019shape-based,Guan2018Application,Dhanachandra2017survey, Karim2020Deep, Marchetti2018Spatial}.
The shortage of labels for training is a big problem in some machine learning applications, such as medical image analysis, and applications of big data \cite{Nasraoui2019Clustering} because labeling is expensive \cite{Hoo-Chang2016Deep}. Since unsupervised machine learning does not use labels for training, to apply cluster analysis can avoid the problem.

In general, the main methods of cluster analysis can be categorized into centroid-based (\textit{e.g.}, k-means), distribution-based (\textit{e.g.}, EM algorithm \cite{Byrne2017EM}), density-based (\textit{e.g.}, DBSCAN \cite{Ester1996density-based}), hierarchical (\textit{e.g.}, Ward linkage \cite{Ward_Jr1963Hierarchical}), and spectral clustering \cite{Von_Luxburg2007tutorial}.
None of the clustering methods, however, is able to perform well with all kinds of datasets \cite{Kleinberg2003Impossibility,Von_Luxburg2012Clustering:}. That is, a clustering method that performs well with some types of datasets would perform poorly with some others. For this reason, various clustering methods have been applied to datasets. Consequently, effective clustering validations (measures of clustering quality) are required to evaluate which clustering method performs well for a dataset \cite{Ben-David2009Measures,Adolfsson2019To}. And, clustering validations are also used to tune the parameters of clustering algorithms.

There are two categories of clustering validations: \textit{internal} and \textit{external} validations. \textit{External validations} use the truth-labels of classes and predicted labels, and \textit{internal validations} use predicted labels and data. Since external validations require true labels and there are no true class labels in unsupervised learning tasks, we can employ only the internal validations in cluster analysis \cite{Liu2013Understanding}.
In fact, to evaluate clustering results by internal validations has the same difficulty as to do clustering itself because measures have no more information than the clustering methods \cite{Pfitzner2008Characterization}. Therefore, the difficulty of designing an internal \textit{Cluster Validity Index} (CVI) is like creating a clustering algorithm. The different part is that a clustering algorithm can update clustering results by a value (loss) from the optimizing function but the CVI provides only a value for clusters evaluation.

\subsection{Related works}
Various CVIs have been created for the clustering of many types of datasets \cite{Desgraupes2017Clustering}. By methods of calculation \cite{Hu2019Internal}, the internal CVIs are based on two categories of representatives: center and non-center. Center-based internal CVIs use descriptors of clusters. For example, the \textit{Davies–Bouldin} index (DB) \cite{Davies1979Cluster} uses cluster diameters and the distance between cluster centroids. Non-center internal CVIs use descriptors of data points. For example, the \textit{Dunn} index \cite{Dunn1974Well-Separated} considers the minimum and maximum distances between two data points.

Besides the DB and Dunn indexes, in this paper, some other typical internal CVIs are selected for comparison. The \textit{Calinski-Harabasz} index (CH) \cite{Calinski1974dendrite} and \textit{Silhouette coefficient} (Sil) \cite{Rousseeuw1987Silhouettes:} are two traditional internal CVIs. In recently developed internal CVIs, the \textit{I} index \cite{Maulik2002Performance}, \textit{WB} index \cite{Zhao2014WB-index:}, \textit{Clustering Validation index based on Nearest Neighbors} (CVNN) \cite{Liu2013Understanding}, and \textit{Cluster Validity index based on Density-involved Distance} (CVDD) \cite{Hu2019Internal} are selected. Eight typical internal CVIs, which range from early studies (Dunn, 1974) to the most recent studies (CVDD, 2019), are selected to compare with our proposed CVI.

In addition, an external CVI - the \textit{Adjusted Rand Index} (ARI) \cite{Santos2009On} is selected as the ground truth for comparison because external validations use the true class labels and predicted labels. Unless otherwise indicated, \textbf{the ``CVIs'' that appear hereafter mean internal CVIs and the only external CVI is named ``ARI''}.

\section{Distance-based Separability Measure}
Since the goal of clustering is to separate a dataset into clusters, in the macro-perspective, how well a dataset has been separated could be indicated via the separability of clusters. In a dataset, data points are assigned class labels by the clustering algorithm. The most difficult situation for separation of the dataset occurs when all labels are randomly assigned and the data points of different classes will have the same distribution (distributions have the same shape, position, and support, \textit{i.e.}, the same probability density function). To analyze the distributions of different-class data, we propose the \textit{Distance-based Separability Index} (DSI) \footnote{More studies about the DSI will appear in other forthcoming publications, which can be found in author’s website linked up with the ORCID: \url{https://orcid.org/0000-0002-3779-9368}.}.

Suppose a dataset contains two classes $X$, $Y$ and have $N_x$, $N_y$ data points, respectively, we can define:

\begin{definition} \label{def:1}
The \textit{Intra-Class Distance} (ICD) set is a set of distances between any two points in the same class. \textit{e.g.}, for class $X$, its ICD set $\{d_x\}$:
\[
\{d_x\}=\{ \|x_i-x_j\|_2 | x_i,x_j\in X;x_i\neq x_j\}.
\]
\end{definition}

\begin{remark}
The metric for distance is Euclidean $(l^2\,\text{-norm})$. Given $|X|=N_x$, then $|\{d_x\}|=\frac{1}{2}N_x(N_x-~1)$.
\end{remark}

\begin{definition} \label{def:2}
The \textit{Between-Class Distance} (BCD) set is the set of distances between any two points from different classes. \textit{e.g.}, for class $X$ and $Y$, their BCD set $\{ d_{x,y}\}$:
\[
\{ d_{x,y} \}=\{ \|x_i-~y_j\|_2 \, |\, x_i\in X;y_j\in Y \}.
\]
\end{definition}

\begin{remark}
Given $|X|=N_x,|Y|=N_y$, then $|\{d_{x,y} \}|=N_x N_y$.
\end{remark}

Then, the Theorem \ref{thm:1} shows how the ICD and BCD sets are related to the distributions of data:

\begin{theorem} \label{thm:1}
When $|\{d_x\}|,|\{d_y\}|\to \infty$, if and only if the two classes $X$ and $Y$ have the same distribution, the distributions of the ICD and BCD sets are identical.
\end{theorem}

The full proof of Theorem \ref{thm:1} is shown in \ref{sec:proof}. Here we provide an informal explanation: points in $X$ and $Y$ having the same distribution and covering the same region can be considered to have been sampled from one distribution $Z$. Hence, both ICDs of $X$ and $Y$, and BCDs between $X$ and $Y$ are actually ICDs of $Z$. Consequently, the distributions of ICDs and BCDs are identical. In other words, that the distributions of the ICD and BCD sets are identical indicates all labels are assigned randomly and thus, the dataset has the least separability.

\subsection{Computation of the DSI}

For computation of the DSI of the two classes $X$ and $Y$, first, the ICD sets of $X$ and $Y$: $\{d_x\},\{d_y\}$ and the BCD set: $\{d_{x,y}\}$ are computed by their definitions (Def. \ref{def:1} and \ref{def:2}). Second, the \textit{Kolmogorov-Smirnov} (KS) distance \cite{scipy.stats.kstest} is applied to examine the similarity of the distributions of the ICD and BCD sets. Although there are other statistical measures to compare two distributions, such as Bhattacharyya distance, Kullback-Leibler divergence, and Jensen-Shannon divergence, most of them require the two sets to have the same number of data points. It is easy to show that the $|\{d_x\}|,|\{d_y\}|$ and $|\{d_{x,y}\}|$ cannot be the same. The Wasserstein distance \cite{ramdas2017wasserstein} is also a potentially suitable measure, but our testing indicates that the Wasserstein distance is not as sensitive as the KS distance. The result of a two-sample KS distance is the maximum distance between two cumulative distribution functions (CDFs):
\[
KS(P,Q)=\sup_{x} |P(x)-Q(x)|
\]
Where $P$ and $Q$ are the respective CDFs of the two distributions $p$ and $q$. 

Hence, the similarities between the ICD and BCD sets are then computed using the KS distance \footnote{By using the \texttt{scipy.stats.ks\_2samp} from the SciPy package in Python. \url{https://docs.scipy.org/doc/scipy/reference/generated/scipy.stats.ks_2samp.html}}: $s_x=KS(\{d_x\},\{d_{x,y}\})$ and $s_y=KS(\{d_y\},\{d_{x,y}\})$. Since there are two classes, the DSI is the average of the two KS distances: $DSI(\{X,Y\})=\frac{(s_x+s_y)}{2}$. The $KS(\{d_x\},\{d_y\})$ is not used because it shows only the shape difference between the distributions of two classes $X$ and $Y$, not their location information. For example, the two distributions of classes $X$ and $Y$ having the same shape, but no overlap will have \textit{zero} KS distance between their ICD sets: $KS(\{d_x\},\{d_y\})=0$. And we do not use the weighted average because once the distributions of the ICD and BCD sets can be well characterized, the sizes of $X$ and $Y$ will not affect the KS distances $s_x$ and $s_y$.

In general, for an $n$-class dataset, we obtain its DSI with the \textbf{DSI Algorithm}:
\begin{enumerate}
    \item Compute $n$ ICD sets for each class: $\{d_{C_i}\};\; i=1,2,\cdots,n$.
    \item Compute $n$ BCD sets for each class. For the $i$-th class of data $C_i$, the BCD set is the set of distances between any two points in $C_i$ and $\overline{C_i}$ (other classes, not $C_i$): $\{d_{C_i,\overline{C_i}}\}$.
    \item Compute the $n$ KS distances between ICD and BCD sets for each class: $s_i=KS(\{d_{C_i }\},\{d_{C_i,\overline{C_i}}\})$.
    \item Calculate the average of the $n$ KS distances; the DSI of this dataset is $DSI(\{C_i \})=\frac{\sum s_i}{n}$.
\end{enumerate}

The running time of computing ICD and BCD sets is linear with the dimensionality and quadratic with the amount of data.

\subsection{Cluster validation using DSI}
A small DSI (low separability) means that the ICD and BCD sets are very similar. In this case, by Theorem \ref{thm:1}, the distributions of classes $X$ and $Y$ are similar too. Hence, data of the two classes are difficult to separate.

\begin{figure}[th]
\centerline{\includegraphics[width=0.65\textwidth]{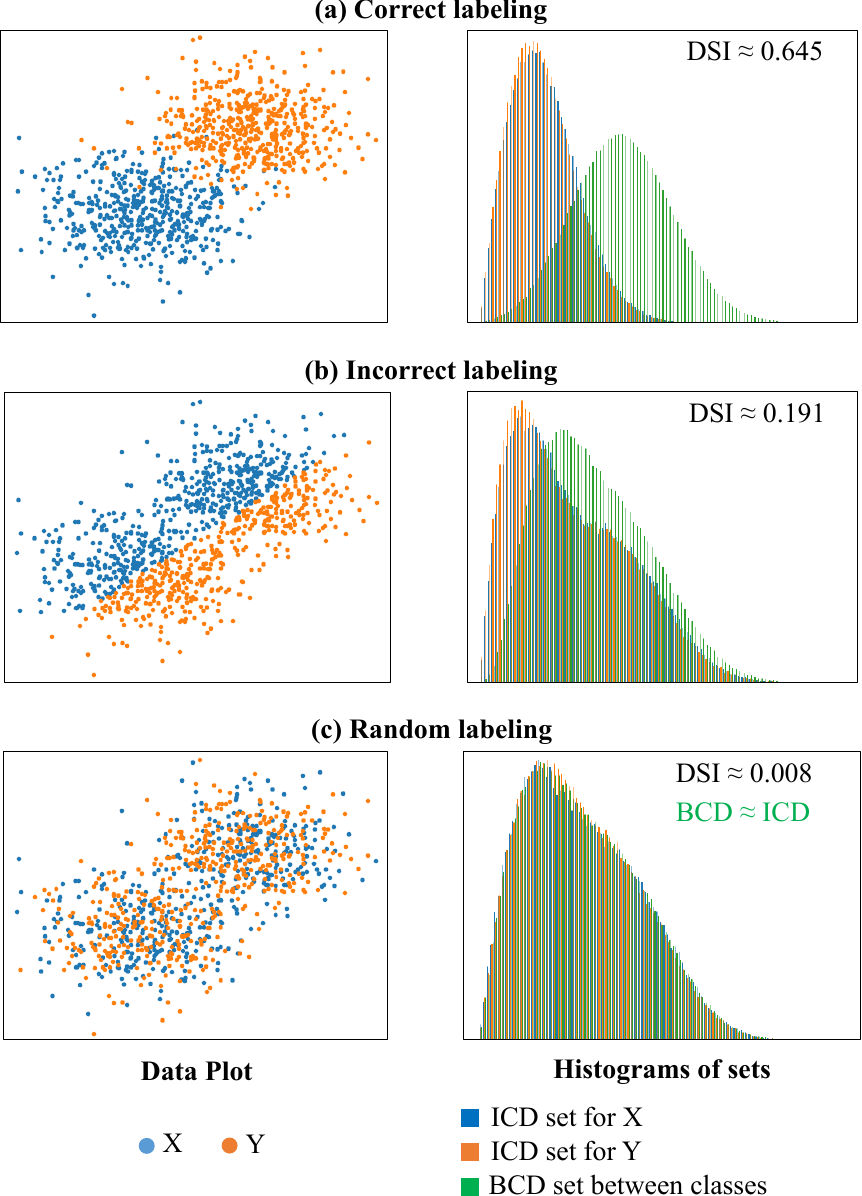}}
\vspace*{8pt}
\caption{Two clusters (classes) datasets with different label assignments. Each histogram indicates the relative frequency of the value of each of the three distance measures (indicated by color).}
\label{fig:1}
\end{figure}

An example of two-class dataset is shown in Figure \ref{fig:1}. Figure \ref{fig:1}a shows that, if the labels are assigned correctly by clustering, the distributions of ICD sets will be different from the BCD set and the DSI will reach the maximum value for this dataset because the two clusters are well separated. 
For an incorrect clustering, in Figure \ref{fig:1}b, the difference between distributions of ICD and BCD sets becomes smaller so that the DSI value decreases.
Figure \ref{fig:1}c shows an extreme situation, that is, if all labels are randomly assigned, the distributions of the ICD and BCD sets will be nearly identical. It is the worst case of separation for the two-class dataset and its separability (DSI) is close to zero. Therefore, the separability of clusters can be reflected well by the proposed DSI. The DSI ranges from 0 to 1, $DSI\in (0,1)$, and \textit{\textbf{we suppose that}} the greater DSI value means the dataset is clustered better.
\section{Materials}
\subsection{Compared CVIs}
CVIs are used to evaluate the clustering results. In this paper, several internal CVIs including the proposed DSI have been employed to examine the clustering results from different clustering methods (algorithms). To use different clustering methods on a given dataset may obtain different cluster results and thus, CVIs are used to select the best clusters. We choose eight commonly used (classical and recent) internal CVIs and an external CVI - the \textit{Adjusted Rand Index} (ARI) to compare with our proposed DSI (Table \ref{tab:1}). The role of ARI is the ground truth for comparison because ARI involves true labels (clusters) of the dataset.

\begin{table}
\tbl{Compared CVIs.}
{\begin{tabular}{ccc}  \toprule
\textbf{Name} & \textbf{Optimal$^a$} & \textbf{Reference} \\  \colrule
\textbf{ARI}$^b$ & MAX & (Santos \& Embrechts, 2009)  \cite{Santos2009On} \\ 
\colrule
\textbf{Dunn} index & MAX & (Dunn, J.,1973) \cite{Dunn1974Well-Separated} \\
\textbf{C}alinski-\textbf{H}arabasz Index & MAX & (Calinski \& Harabasz, 1974) \cite{Calinski1974dendrite} \\
\textbf{D}avies–\textbf{B}ouldin index & min & (Davies \& Bouldin, 1979) \cite{Davies1979Cluster} \\
\textbf{Sil}houette Coefficient & MAX & (Rousseeuw, 1987) \cite{Rousseeuw1987Silhouettes:} \\
\textbf{I} & MAX & (U. Maulik, 2002) \cite{Maulik2002Performance} \\
\textbf{CVNN} & min & (Yanchi L., 2013)  \cite{Liu2013Understanding} \\
\textbf{WB} & min & (Zhao Q., 2014) \cite{Zhao2014WB-index:} \\
\textbf{CVDD} & MAX & (Lianyu H., 2019) \cite{Hu2019Internal} \\
\textbf{DSI} & MAX & Proposed \\ \botrule
\end{tabular}}
\begin{tabnote}
$^{a.}$ Optimal column means the CVI for best case has the minimum or maximum value.
$^{b.}$ The ground truth for comparison.
\end{tabnote}
\label{tab:1}
\end{table}

\subsection{Synthetic and real datasets}
\label{sec:datasets}
In this paper, the synthetic datasets for clustering are from the Tomas Barton repository \footnote{\url{  https://github.com/deric/clustering-benchmark/tree/master/src/main/resources/datasets/artificial}}, which contains 122 artificial datasets. Each dataset has hundreds to thousands of objects with several to tens of classes in two or three dimensions (features). We have selected 97 datasets for experiment because the 25 unused datasets have too many objects to run the clustering processing in reasonable time. The names of the 97 used synthetic datasets are shown in \ref{sec:syn_names}. Illustrations of these datasets can be found in Tomas Barton's homepage \footnote{\url{https://github.com/deric/clustering-benchmark}}.

The 12 real datasets used for clustering are from three sources: the \textit{sklearn.datasets} package \footnote{\url{https://scikit-learn.org/stable/datasets}}, UC Irvine Machine Learning Repository \cite{Dheeru2017UCI} and Tomas Barton's repository (real world datasets) \footnote{\url{https://github.com/deric/clustering-benchmark/tree/master/src/main/resources/datasets/real-world}}. Unlike the synthetic datasets, the dimensions (feature numbers) of most selected real datasets are greater than three. Hence, CVIs must be used to evaluate their clustering results rather than plotting clusters as for 2D or 3D synthetic datasets. Details about the 12 real datasets appear in Table \ref{tab:2}.

\begin{table}
\tbl{The description of used real datasets.}
{\begin{tabular}{ccccc} \toprule
\textbf{Name} & \textbf{Title} & \textbf{Object\#} & \textbf{Feature\#} & \textbf{Class\#} \\ \colrule
Iris & Iris plants dataset & 150 & 4 & 3 \\ 
digits & Optical recognition of handwritten digits dataset & 5620 & 64 & 10 \\ 
wine & Wine recognition dataset & 178 & 13 & 3 \\ 
cancer & Breast cancer Wisconsin (diagnostic) dataset & 569 & 30 & 2 \\ 
faces & Olivetti faces dataset & 400 & 4096 & 40 \\ 
vertebral & Vertebral column data & 310 & 6 & 3 \\ 
haberman & Haberman's survival data & 306 & 3 & 2 \\ 
sonar & Sonar, Mines vs. Rocks & 208 & 60 & 2 \\ 
tae & Teaching Assistant evaluation & 151 & 5 & 3 \\ 
thy & Thyroid disease data  & 215 & 5 & 3 \\ 
vehicle & Vehicle silhouettes & 946 & 18 & 4 \\ 
zoo & Zoo data  & 101 & 16 & 7 \\ \botrule
\end{tabular}}
\label{tab:2}
\end{table}

\section{Experiments}
In general, there are two strategies to evaluate CVIs using a dataset: 1) to compare with ground truth (real clusters with labels); 2) to predict the number of clusters (classes) by finding the optimal number of clusters as identified by CVIs \cite{Cheng2019Novel}.

\subsection{Using real clusters}
By using datasets' information of real clusters with labels, the steps to evaluate CVIs are:
\begin{enumerate}
\item To obtain clustering results by running different clustering methods (algorithms) on a dataset.
\item To compute CVIs of these clustering results and their ARI (ground truth) using real labels.
\item To compare the values of CVI with ARI.
\item To repeat the former three steps for a new dataset.
\end{enumerate}

In this paper, five clustering algorithms from various categories are used, they are: k-means, Ward linkage, spectral clustering, BIRCH \cite{Zhang1996BIRCH:} and EM algorithm (Gaussian Mixture). The CVIs used for evaluation and comparison are shown in Table \ref{tab:1} and the used datasets are introduced in Section \ref{sec:datasets}. And we provide two evaluation methods to compare the values of CVIs with the ground truth ARI; they are called \textit{Hit-the-best} and \textit{Rank-difference}, which are described as follows.

\subsubsection{Evaluation metric: Hit-the-best}
For a dataset, clustering results obtained by different clustering algorithms would have different CVIs and ARI. If a CVI gives the best score to a clustering result that also has the best ARI score, this CVI is considered to be a correct prediction (hit-the-best).
Table \ref{tab:3} shows CVIs of clustering results by different clustering methods on a dataset.
For the \texttt{wine} dataset, k-means receives the best ARI score and Dunn, DB, WB, I, CVNN and DSI give k-means the best score; and thus, the six CVIs are hit-the-best. If we mark hit-the-best CVIs as 1 and others as 0, CVI scores in Table \ref{tab:3} can be converted to hit-the-best results (Table \ref{tab:4}) for the \texttt{wine} dataset.

\begin{table}
\tbl{CVI scores of clustering results on the \texttt{wine} recognition dataset.}
{\begin{tabular}{cccccc}  \toprule
\diagbox{\textbf{Validity}$^a$}{\textbf{\makecell[r]{Clustering \\ method}}} & \textbf{KMeans} & \textbf{\makecell{Ward \\ Linkage}} & \textbf{\makecell{Spectral \\ Clustering}} & \textbf{BIRCH} & \textbf{EM} \\
\colrule
ARI$^b$ + & \textbf{0.913}$^c$ & 0.757 & 0.880 & 0.790 & 0.897 \\
\colrule
Dunn + & \textbf{0.232} & 0.220 & 0.177 & 0.229 & \textbf{0.232} \\
CH + & 70.885 & 68.346 & 70.041 & 67.647 & \textbf{70.940} \\
DB - & \textbf{1.388} & 1.390 & 1.391 & 1.419 & 1.389 \\
Silhouette + & 0.284 & 0.275 & 0.283 & 0.278 &\textbf{0.285} \\
WB - & \textbf{3.700} & 3.841 & 3.748 & 3.880 & \textbf{3.700} \\
I + & \textbf{5.421} & 4.933 & 5.326 & 4.962 & \textbf{5.421} \\
CVNN - & \textbf{21.859} & 22.134 & 21.932 & 22.186 & \textbf{21.859} \\
CVDD + & 31.114 & \textbf{31.141} & 29.994 & 30.492 & 31.114 \\
DSI + & \textbf{0.635} & 0.606 & 0.629 & 0.609 & 0.634 \\
\botrule
\end{tabular}}
\begin{tabnote}
$^{a.}$ CVI for best case has the minimum (-) or maximum (+) value.
$^{b.}$ The first row shows results of ARI as ground truth; other rows are CVIs.
$^{c.}$ Bold value: the best case by the measure of this row.
\end{tabnote}
\label{tab:3}
\end{table}

\begin{table}
\tbl{Hit-the-best results for the \texttt{wine} dataset.}
{\begin{tabular}{cccccccccc}  \toprule
\diagbox[width=8em]{\textbf{\scalebox{0.95}{Dataset}}}{\textbf{\scalebox{0.95}{CVI}}} & \textbf{Dunn} & \textbf{CH} & \textbf{DB} & \textbf{Sil$^a$} & \textbf{WB} & \textbf{I} & \textbf{CVNN} & \textbf{CVDD} & \textbf{DSI} \\
\colrule
wine & 1 & 0 & 1 & 0 & 1 & 1 & 1 & 0 & 1 \\
\botrule
\end{tabular}}
\begin{tabnote}
$^{a.}$ Sil = Silhouette.
\end{tabnote}
\label{tab:4}
\end{table}

For the hit-the-best, however, the best score can be unstable and random in some cases. For example, in Table \ref{tab:3}, the ARI score of EM is very close to that of k-meams and the Silhouette score of EM is also very close to that of k-meams. If these values fluctuated a little and changed the best cases, the comparison outcome for this dataset will be changed. Another drawback of hit-the-best is that it concerns only one best case and ignores others; it does not evaluate the whole picture for one dataset. The hit-the-best might be a stricter criterion but lacks robustness, and it is vulnerable to extreme cases such as when scores of different clustering results are very close to each other. Hence, we create another method to compare the score sequences of CVIs and ARI through their ranks.

\subsubsection{Evaluation metric: Rank-difference}
This comparison method fixes the two problems of the hit-the-best: one is instability for similar scores and the other is the bias on only one case.

We apply quantization to solve the problem of similar scores. Every score in the score sequence of a CVI (\textit{i.e.}, a row in Table \ref{tab:3}) will be assigned a rank number and similar scores have high probability to be allocated the same rank number. The procedure is:
\begin{enumerate}
\item Find the minimum and maximum values of $N$ scores from one sequence.
\item Uniformly divide [min,max] into $N-1$ intervals.
\item Label intervals from max to min by $1,2,\ldots,N-1$.
\item If a score is in the $k$-th interval, its rank number is $k$.
\item Define rank number of max is 1, and intervals are left open and right closed: (upper value, lower value].
\end{enumerate}

\begin{figure}[th]
\centerline{\includegraphics{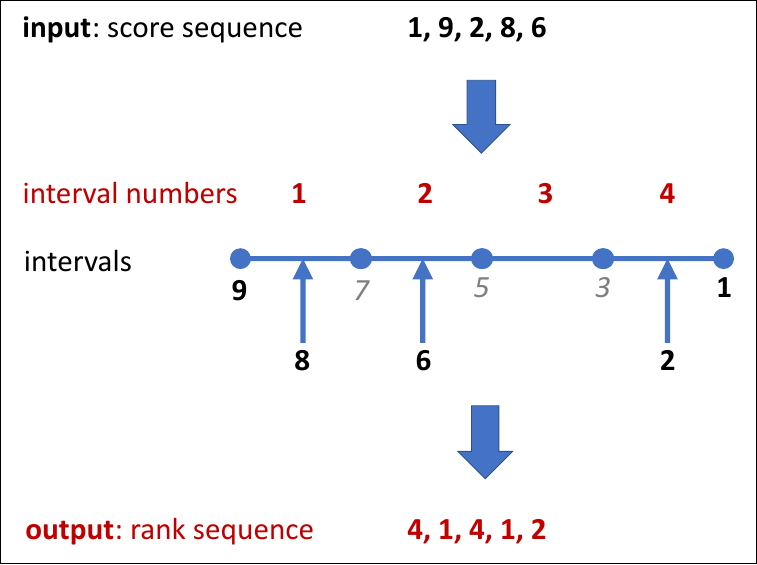}}
\vspace*{8pt}
\caption{An example of rank numbers assignment.}
\label{fig:2}
\end{figure}

Figure \ref{fig:2} shows an example of converting a \textit{score sequence} to a \textit{rank sequence} (rank numbers). The rank number of scores 9 and 8 is 1 because they are in the 1st interval. For the same reason, the rank number of scores 1 and 2 is 4. Such quantization is better than assigning rank numbers by ordering because it avoids the assignment of different rank numbers to very close scores in most cases (it is still possible to use different rank numbers for very close scores; for example, in the Figure \ref{fig:2} case, if scores 8 and 6 changed to 7.1 and 6.9, their rank numbers will still be 1 and 2 even they are very close).

\begin{remark}
If the score whose rank number is 1 (1-rank score) represents the optimal performance, that the rank number of the maximum CVI score is 1 only works for the CVI whose optimum is maximum but does not work for the CVI whose optimum is minimum, like DB and WB, because its 1-rank score should be minimum.
A simple solution to make the rank number work for both types of CVIs is to \textbf{\textit{negate}} all values in score sequences of the CVIs whose optimum is minimum before converting to rank sequence (Figure \ref{fig:2}). Thus, the 1-rank score always represents the optimal performance for all CVIs.
\end{remark}

Table \ref{tab:5} shows rank sequences of CVIs converted from the score sequences in Table \ref{tab:3}. For each CVI, four ranks are assigned to five scores. Since the ARI row shows the truth rank sequence, for rank sequences in other CVI rows, the more similar to the ARI row, the better the CVI performs.

\begin{table}
\tbl{Rank sequences of CVIs converted from the score sequences in Table \ref{tab:3}.}
{\begin{tabular}{cccccc}  \toprule
\diagbox[width=10.5em]{\textbf{\scalebox{0.95}{Validity}}}{\scalebox{0.95}{\textbf{\makecell[r]{Clustering \\ method}}}} & \scalebox{0.95}{\textbf{KMeans}} & \scalebox{0.95}{\textbf{\makecell{Ward \\ Linkage}}} & \scalebox{0.95}{\textbf{\makecell{Spectral \\ Clustering}}} & \scalebox{0.95}{\textbf{BIRCH}} & \scalebox{0.95}{\textbf{EM}} \\
\colrule
ARI$^a$ & 1 & 4 & 1 & 4 & 1 \\
\colrule
Dunn & 1 & 1 & 4 & 1 & 1 \\
CH & 1 & 4 & 2 & 4 & 1 \\
DB & 1 & 1 & 1 & 4 & 1 \\
Silhouette & 1 & 4 & 1 & 3 & 1 \\
WB & 1 & 4 & 2 & 4 & 1 \\
I & 1 & 4 & 1 & 4 & 1 \\
CVNN & 1 & 4 & 1 & 4 & 1 \\
CVDD & 1 & 1 & 4 & 3 & 1 \\
DSI & 1 & 4 & 1 & 4 & 1 \\
\botrule
\end{tabular}}
\begin{tabnote}
$^{a.}$ The first row shows results of ARI as ground truth; other rows are CVIs.
\end{tabnote}
\label{tab:5}
\end{table}

For two score sequences (\textit{e.g.}, CVI and ARI), after quantizing them to two rank sequences, we will compute the difference of two rank sequences (called \textit{rank-difference}), which is simply defined as the summation of absolute difference between two rank sequences. For example, the two rank sequences from Table \ref{tab:5} are:
\[
\begin{array}{lr}
     ARI: & \{1,\ 4,\ 1,\ 4,\ 1\} \\
     CVDD:& \{1,\ 1,\ 4,\ 3,\ 1\}
\end{array}
\]
Their rank-difference, which is the summation of absolute difference, is:
\[
\left|1-1\right|+\left|4-1\right|+\left|1-4\right|+\left|4-3\right|+\left|1-1\right|=7
\]
Smaller rank-difference means the distance of two sequences is closer. That two sequences of CVI and ARI are closer indicates a better prediction. It is not difficult to show that rank-difference for two $N$-length score sequences lies in the ranges $[0,N(N-2)]$. Table \ref{tab:6} shows rank-differences calculated by the ARI and nine CVIs from Table \ref{tab:5}. The CVI having the lower rank-difference value is better and 0 is the best because it has the same performance as the ground truth (ARI).

\begin{table}
\tbl{Rank-difference results for the \texttt{wine} dataset.}
{\begin{tabular}{cccccccccc}  \toprule
\diagbox[width=8em]{\textbf{\scalebox{0.95}{Dataset}}}{\textbf{\scalebox{0.95}{CVI}}} & \textbf{Dunn} & \textbf{CH} & \textbf{DB} & \textbf{Sil$^a$} & \textbf{WB} & \textbf{I} & \textbf{CVNN} & \textbf{CVDD} & \textbf{DSI} \\
\colrule
wine & 9 & 1 & 3 & 1 & 1 & 0 & 0 & 7 & 0 \\
\botrule
\end{tabular}}
\begin{tabnote}
$^{a.}$ Sil = Silhouette.
\end{tabnote}
\label{tab:6}
\end{table}

\subsection{To predict the number of clusters}
\label{sec:c_num}
Some clustering methods require setting the number of clusters (classes) in advance, such as the k-means, spectral clustering, and Gaussian mixture (EM). Suppose we have a dataset and know its real number of clusters, $c$; then the steps to evaluate CVIs through predicting the number of clusters in this dataset are:
\begin{enumerate}
\item To run clustering algorithms by setting the number of clusters $k=2,3,4,\ldots$ (the real number of clusters $c$ is included) to get clusters.
\item To compute CVIs of these clusters.
\item The predicted number of clusters by the $i$-th CVI: $\hat{k}_i$, is the number of clusters that perform best on the $i$-th CVI. (\textit{i.e.}, the optimal number of clusters recognized by this CVI)
\item The successful prediction of the $i$-th CVI is that its predicted number of clusters equals the real number of clusters: $\hat{k}_i=c$.
\end{enumerate}

For several CVIs, the number of successful predictions could be zero, one, two, or more. Besides CVIs, the success also depends on the datasets and clustering methods. In this study, we selected the \texttt{wine}, \texttt{tae}, \texttt{thy}, and \texttt{vehicle} datasets (see Table \ref{tab:2}), and clustering methods: k-means, spectral clustering, and EM algorithm.
\section{Results}
\subsection{Clusters of real and synthetic datasets}
\begin{table}[t]
\tbl{Hit-the-best results for real datasets.}
{\begin{tabular}{cccccccccc}  \toprule
\diagbox[width=8em]{\textbf{\scalebox{0.95}{Dataset}}}{\textbf{\scalebox{0.95}{CVI}}} & \textbf{Dunn} & \textbf{CH} & \textbf{DB} & \textbf{Sil$^a$} & \textbf{WB} & \textbf{I} & \textbf{CVNN} & \textbf{CVDD} & \textbf{DSI} \\
\colrule
Iris & 0 & 0 & 0 & 0 & 0 & 0 & 0 & 1 & 0\\ 
digits & 0 & 0 & 0 & 1 & 0 & 0 & 1 & 0 & 1\\ 
wine & 1 & 0 & 1 & 0 & 1 & 1 & 1 & 0 & 1\\ 
cancer & 0 & 0 & 0 & 0 & 0 & 0 & 1 & 0 & 0\\ 
faces & 1 & 1 & 1 & 1 & 1 & 1 & 0 & 1 & 1\\ 
vertebral & 0 & 0 & 0 & 0 & 0 & 0 & 0 & 0 & 0\\ 
haberman & 0 & 1 & 0 & 0 & 1 & 0 & 0 & 0 & 0\\ 
sonar & 0 & 1 & 0 & 0 & 1 & 0 & 0 & 0 & 0\\ 
tae & 0 & 0 & 0 & 0 & 0 & 0 & 1 & 1 & 0\\ 
thy & 0 & 0 & 0 & 0 & 0 & 0 & 0 & 0 & 0\\ 
vehicle & 0 & 0 & 0 & 0 & 0 & 0 & 1 & 0 & 1\\ 
zoo & 1 & 0 & 1 & 0 & 0 & 1 & 0 & 0 & 1\\ 
\colrule
Total$^b$   & 3  & 3  & 3  & 2  & 4  & 3  & 5  & 3  & 5 \\ 
(\textbf{rank})  & (\textbf{4}) &  (\textbf{4}) &  (\textbf{4}) &  (\textbf{9}) &  (\textbf{3}) &  (\textbf{4}) &  (\textbf{1}) &  (\textbf{4}) &  (\textbf{1})\\ 
\botrule
\end{tabular}}
\begin{tabnote}
$^{a.}$ Sil = Silhouette.
$^{b.}$ Larger value is better (rank number is smaller).
\end{tabnote}
\label{tab:7}
\end{table}
\begin{table}[t]
\tbl{Rank-difference results for real datasets.}
{\begin{tabular}{cccccccccc}  \toprule
\diagbox[width=8em]{\textbf{\scalebox{0.95}{Dataset}}}{\textbf{\scalebox{0.95}{CVI}}} & \textbf{Dunn} & \textbf{CH} & \textbf{DB} & \textbf{Sil$^a$} & \textbf{WB} & \textbf{I} & \textbf{CVNN} & \textbf{CVDD} & \textbf{DSI} \\
\colrule
Iris & 8 & 13 & 15 & 15 & 13 & 11 & 15 & 6 & 15\\ 
digits & 2 & 2 & 1 & 1 & 4 & 6 & 8 & 7 & 6\\ 
wine & 9 & 1 & 3 & 1 & 1 & 0 & 0 & 7 & 0\\ 
cancer & 8 & 7 & 6 & 9 & 7 & 8 & 2 & 7 & 9\\ 
faces & 4 & 3 & 4 & 4 & 2 & 3 & 9 & 2 & 5\\ 
vertebral & 6 & 13 & 14 & 12 & 15 & 13 & 15 & 6 & 13\\ 
haberman & 9 & 7 & 7 & 7 & 7 & 9 & 7 & 7 & 8\\ 
sonar & 7 & 3 & 3 & 4 & 3 & 4 & 11 & 10 & 3\\ 
tae & 9 & 14 & 9 & 9 & 14 & 15 & 0 & 9 & 9\\ 
thy & 5 & 2 & 2 & 2 & 2 & 6 & 2 & 3 & 10\\ 
vehicle & 12 & 11 & 9 & 13 & 13 & 12 & 3 & 3 & 7\\ 
zoo & 1 & 6 & 1 & 6 & 6 & 1 & 9 & 8 & 1\\ 
\colrule
Total$^b$ & 80 & 82 & 74 & 83 & 87 & 88 & 81 & 75 & 86 \\ 
(\textbf{rank}) &  (\textbf{3}) &  (\textbf{5}) &  (\textbf{1}) &  (\textbf{6}) & (\textbf{8}) &  (\textbf{9}) &  (\textbf{4}) & (\textbf{2}) & (\textbf{7})\\ 
\botrule
\end{tabular}}
\begin{tabnote}
$^{a.}$ Sil = Silhouette.
$^{b.}$ Smaller value is better (rank number is smaller).
\end{tabnote}
\label{tab:8}
\end{table}

As discussed before, for one dataset and a CVI, an evaluation result can be computed by using the hit-the-best or rank-difference metric. In other words, one result is obtained by comparing one CVI row in Table \ref{tab:3} with the ground truth (ARI). The outcome of a hit-the-best comparison is either 0 or 1; 1 means that the best clusters predicted by CVI are the same as ARI; otherwise, the outcome is 0. Table \ref{tab:4} shows the hit-the-best results of nine CVIs on the \texttt{wine} dataset. The outcome of the rank-difference comparison is a value in the range $[0,N(N-2)]$, where $N$ is the sequence length. As Table \ref{tab:5} shows, the length of sequences is 5; hence, the range of rank-difference is [0, 15]. Table \ref{tab:6} shows the rank-difference results of nine CVIs on the \texttt{wine} dataset. The smaller rank-difference value means the CVI predicts better.

We applied\footnote{The code can be found in author’s website linked up with the ORCID: \url{https://orcid.org/0000-0002-3779-9368}} the evaluation method to the selected CVIs (Table \ref{tab:1}) by using real and synthetic datasets (Section \ref{sec:datasets}) and the five clustering methods (Table \ref{tab:3}). Table \ref{tab:7} and Table \ref{tab:9} are hit-the-best comparison results for real and synthetic datasets. Table \ref{tab:8} and Table \ref{tab:10} are rank-difference comparison results for real and synthetic datasets. To compare across data sets, we summed all results at the bottom of each table. For the hit-the-best comparison, the larger total value is better because more hits appear. For the rank-difference comparison, the smaller total value is better because results of the CVI are closer to that of ARI. Finally, ranks in the last row uniformly indicate CVIs’ performances. The smaller rank number means better performance. Since there are 97 synthetic datasets, to keep the tables to manageable lengths, Tables \ref{tab:9} and \ref{tab:10} present illustrative values for the datasets and most importantly, the totals and ranks for each measure.

\begin{table}[t]
\tbl{Hit-the-best results for 97 synthetic datasets.}
{\begin{tabular}{cccccccccc}  \toprule
\diagbox[width=8em]{\textbf{\scalebox{0.95}{Dataset}}}{\textbf{\scalebox{0.95}{CVI}}} & \textbf{Dunn} & \textbf{CH} & \textbf{DB} & \textbf{Sil$^a$} & \textbf{WB} & \textbf{I} & \textbf{CVNN} & \textbf{CVDD} & \textbf{DSI} \\
\colrule
3-spiral & 1 & 0 & 0 & 0 & 0 & 0 & 0 & 1 & 0 \\ 
aggregation & 0 & 0 & 0 & 0 & 0 & 0 & 1 & 1 & 1 \\ 
\vdots & \vdots & \vdots & \vdots & \vdots & \vdots & \vdots & \vdots & \vdots & \vdots \\ 
zelnik5 & 1 & 0 & 0 & 0 & 0 & 0 & 0 & 1 & 0\\ 
zelnik6 & 1 & 1 & 0 & 0 & 1 & 0 & 0 & 0 & 0\\ 
\colrule
Total$^b$ & 46 & 30 & 35 & 35  & 29  & 31 & 35  & 50 & 40 \\ 
 (\textbf{rank}) &  (\textbf{2}) &  (\textbf{8}) & (\textbf{4}) & (\textbf{4}) & (\textbf{9}) & (\textbf{7}) & (\textbf{4}) & (\textbf{1}) & (\textbf{3})\\ 
\botrule
\end{tabular}}
\begin{tabnote}
$^{a.}$ Sil = Silhouette.
$^{b.}$ Larger value is better (rank number is smaller).
\end{tabnote}
\label{tab:9}
\end{table}

\begin{table}[t]
\tbl{Rank-difference results for 97 synthetic datasets.}
{\begin{tabular}{cccccccccc}  \toprule
\diagbox[width=8em]{\textbf{\scalebox{0.95}{Dataset}}}{\textbf{\scalebox{0.95}{CVI}}} & \textbf{Dunn} & \textbf{CH} & \textbf{DB} & \textbf{Sil$^a$} & \textbf{WB} & \textbf{I} & \textbf{CVNN} & \textbf{CVDD} & \textbf{DSI} \\
\colrule
3-spiral & 2 & 12 & 14 & 13 & 14 & 12 & 13 & 1 & 13 \\ 
aggregation & 3 & 3 & 2 & 2 & 4 & 5 & 2 & 5 & 3 \\ 
\vdots & \vdots & \vdots & \vdots & \vdots & \vdots & \vdots & \vdots & \vdots & \vdots \\ 
zelnik5 & 4 & 10 & 12 & 10 & 11 & 11 & 10 & 4 & 11 \\ 
zelnik6 & 4 & 3 & 2 & 2 & 3 & 3 & 5 & 2 & 2 \\ 
\colrule
Total$^b$ & 406 & 541 & 547  & 489 & 583 & 554  & 504  & 337  & 415 \\  
(\textbf{rank}) & (\textbf{2}) & (\textbf{6}) & (\textbf{7}) & (\textbf{4}) & (\textbf{9}) & (\textbf{8}) & (\textbf{5}) & (\textbf{1}) & (\textbf{3}) \\  
\botrule
\end{tabular}}
\begin{tabnote}
$^{a.}$ Sil = Silhouette.
$^{b.}$ Smaller value is better (rank number is smaller).
\end{tabnote}
\label{tab:10}
\end{table}

\subsection{Prediction of number of clusters}
Another strategy of CVI evaluation is to predict the number of clusters (classes). Its detailed processes are described in Section \ref{sec:c_num}. The clustering methods we selected require setting the number of clusters (classes) in advance; they are: k-means, spectral clustering, and EM algorithm. The \textit{a priori} number of clusters we set for the three algorithms are: $k=2,\ 3,\ 4,\ 5,\ 6$ (the real number of clusters is included). Clustering algorithms have been applied on four datasets: the \texttt{wine}, \texttt{tae}, \texttt{thy}, and \texttt{vehicle} datasets (see Table \ref{tab:2} for details).

\begin{table}[h]
\tbl{Number of clusters prediction results on the \texttt{wine}, \texttt{tae}, \texttt{thy}, and \texttt{vehicle} datasets.}{
\begin{tabular}{cc}
\\ 
The \texttt{wine} dataset has 178 samples in 3 classes.
 & The \texttt{tae} dataset has 151 samples in 3 classes. \\
\begin{tabular}{cccc}  \toprule
\diagbox[width=8em]{\scalebox{0.8}{\textbf{Validity}}}{\scalebox{0.8}{\textbf{\makecell[r]{Clustering \\ method}}}} & \scalebox{0.8}{\textbf{KMeans}} & \scalebox{0.8}{\textbf{\makecell{Spectral \\ Clustering}}} & \scalebox{0.8}{\textbf{EM}} \\
\colrule
Dunn & \textbf{3}$^b$ & 4 & 6 \\
CH & \textbf{3} & \textbf{3} & \textbf{3} \\
DB & \textbf{3} & \textbf{3} & \textbf{3} \\
Sil$^a$ & \textbf{3} & \textbf{3} & \textbf{3} \\
WB & \textbf{3} & \textbf{3} & \textbf{3} \\
I & \textbf{3} & 2 & 2 \\
CVNN & 2 & 2 & 2 \\
CVDD & 2 & 2 & 2 \\
DSI & \textbf{3} & \textbf{3} & \textbf{3} \\

\botrule
\end{tabular} &

\begin{tabular}{cccc}  \toprule
\diagbox[width=8em]{\scalebox{0.8}{\textbf{Validity}}}{\scalebox{0.8}{\textbf{\makecell[r]{Clustering \\ method}}}} & \scalebox{0.8}{\textbf{KMeans}} & \scalebox{0.8}{\textbf{\makecell{Spectral \\ Clustering}}} & \scalebox{0.8}{\textbf{EM}} \\
\colrule
Dunn & 2 & \textbf{3} & 2 \\
CH & 6 & 6 & 4 \\
DB & 6 & 6 & 5 \\
Sil & 6 & \textbf{3} & 2 \\
WB & 6 & 6 & 6 \\
I & 5 & \textbf{3} & \textbf{3} \\
CVNN & 2 & 2 & 2 \\
CVDD & 2 & 2 & 2 \\
DSI & 6 & \textbf{3} & 5 \\
\botrule
\end{tabular} \\
\\
\\ 
The \texttt{thy} dataset has 215 samples in 3 classes.
 & The \texttt{vehicle} dataset has 948 samples in 4 classes. \\
\begin{tabular}{cccc}  \toprule
\diagbox[width=8em]{\scalebox{0.8}{\textbf{Validity}}}{\scalebox{0.8}{\textbf{\makecell[r]{Clustering \\ method}}}} & \scalebox{0.8}{\textbf{KMeans}} & \scalebox{0.8}{\textbf{\makecell{Spectral \\ Clustering}}} & \scalebox{0.8}{\textbf{EM}} \\
\colrule
Dunn & 5 & 2 & 6 \\
CH & \textbf{3} & \textbf{3} & \textbf{3} \\
DB & 5 & \textbf{3} & 4 \\
Sil & 4 & \textbf{3} & 2 \\
WB & 6 & \textbf{3} & 6 \\
I & \textbf{3} & \textbf{3} & 4 \\
CVNN & 2 & 2 & 2 \\
CVDD & 2 & 2 & 5 \\
DSI & 5 & \textbf{3} & 6 \\
\botrule
\end{tabular} &

\begin{tabular}{cccc}  \toprule
\diagbox[width=8em]{\scalebox{0.8}{\textbf{Validity}}}{\scalebox{0.8}{\textbf{\makecell[r]{Clustering \\ method}}}} & \scalebox{0.8}{\textbf{KMeans}} & \scalebox{0.8}{\textbf{\makecell{Spectral \\ Clustering}}} & \scalebox{0.8}{\textbf{EM}} \\
\colrule
Dunn & 6 & 2 & 5 \\
CH & 2 & 2 & 2 \\
DB & 2 & 2 & 2 \\
Sil & 2 & 2 & 2 \\
WB & 3 & 3 & 3 \\
I & 5 & 2 & 5 \\
CVNN & 2 & 2 & 2 \\
CVDD & 2 & 2 & 2 \\
DSI & 5 & \textbf{4} & 5 \\
\botrule
\end{tabular} \\
\\ 
\end{tabular}
}
\begin{tabnote}
$^{a.}$ Sil = Silhouette.
$^{b.}$ Bold value: the successful prediction of the CVI whose predicted number of clusters equals the real number of clusters.
\end{tabnote}
\label{tab:11}
\end{table}

Table \ref{tab:11} shows prediction of the number of clusters based on CVIs, clustering algorithms and datasets. The predicted number of clusters by a CVI is the number of clusters that perform best on this CVI. Captions of sub-tables contain the real number of clusters (classes) for each dataset. A successful prediction of the CVI is that its predicted number of clusters equals the real number of clusters. In the results, it is worth noting that \ul{only DSI successfully predicted the number of clusters from spectral clustering for all datasets}. This implies that DSI may work well with the spectral clustering method.

\section{Discussion}
Although DSI obtains only one first-rank (Table \ref{tab:7}) compared with other CVIs in experiments, having no last rank means that it still performs better than some other CVIs. It is worth emphasizing that all compared CVIs are excellent and widely used. Therefore, experiments show that DSI can join them as a new promising CVI. Actually, by examining those CVI evaluation results, we confirm that \textbf{none of the CVIs performs well for all datasets}. And thus, it would be better to measure clustering results by using several effective CVIs. The DSI provides another CVI option. Also, DSI is unique: none of the other CVIs performs the same as DSI. For example, in Table \ref{tab:7}, for the \texttt{vehicle} dataset, only CVNN and DSI predicted correctly. But for \texttt{zoo} dataset, CVNN was wrong and DSI was correct. For another example, in Table \ref{tab:8}, for the \texttt{sonar} dataset, DSI performed better than Dunn, CVNN, and CVDD; but for the \texttt{cancer} dataset, Dunn, CVNN, and CVDD performed better than DSI. More examples of the diversity of CVI are shown in Table \ref{tab:12} and their plots with true labels are shown in Figure \ref{fig:3} (the \texttt{atom} dataset has three features, and the others have two features). 

\begin{table}
\tbl{Rank-difference results for selected synthetic datasets.}
{\begin{tabular}{cccccccccc}  \toprule
\diagbox[width=8em]{\textbf{\scalebox{0.95}{Dataset}}}{\textbf{\scalebox{0.95}{CVI}}} & \textbf{Dunn} & \textbf{CH} & \textbf{DB} & \textbf{Sil$^a$} & \textbf{WB} & \textbf{I} & \textbf{CVNN} & \textbf{CVDD} & \textbf{DSI} \\
\colrule
atom & 0 & 15 & 15 & 15 & 15 & 14 & 4 & 0 & 0 \\
disk-4000n & 10 & 0 & 7 & 0 & 0 & 0 & 11 & 12 & 1 \\
disk-1000n & 6 & 12 & 15 & 12 & 13 & 14 & 15 & 8 & 14 \\
D31 & 5 & 1 & 2 & 1 & 0 & 2 & 10 & 2 & 0 \\
flame & 10 & 6 & 11 & 7 & 7 & 8 & 12 & 11 & 7 \\
square3 & 11 & 0 & 2 & 0 & 0 & 7 & 0 & 11 & 0 \\
\botrule
\end{tabular}}
\begin{tabnote}
$^{a.}$ Sil = Silhouette.
\end{tabnote}
\label{tab:12}
\end{table}

The former examples show the need for employing more CVIs because each is different and every CVI may have its special capability.  That capability, however, is difficult to describe clearly. Some CVIs’ definitions show them to be categorized into center/non-center representative \cite{Hu2019Internal} or density-representative. Similarly, the DSI is a separability-representative CVI. That is, DSI performs better for clusters having high separability with true labels (like the \texttt{atom} dataset in Figure \ref{fig:3}); otherwise, if real clusters have low separability, the incorrectly predicted clusters may have a higher DSI score (Figure \ref{fig:4}).

\begin{figure}
    \centering
    \subfloat[\centering atom]{
    \frame{\includegraphics[height=0.2\textwidth,width=0.2\textwidth]{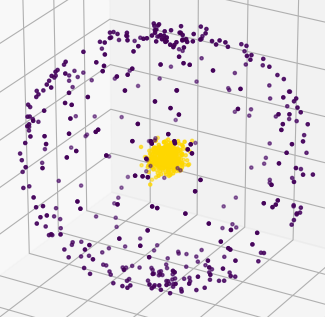}}}
    \quad
    \subfloat[\centering disk-4000n]{
    \frame{\includegraphics[height=0.2\textwidth,width=0.2\textwidth]{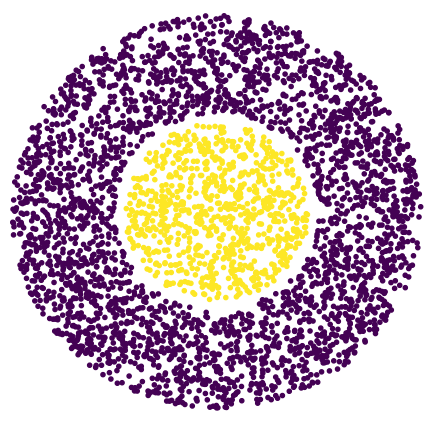}}}
    \quad
    \subfloat[\centering disk-1000n]{
    \frame{\includegraphics[height=0.2\textwidth,width=0.2\textwidth]{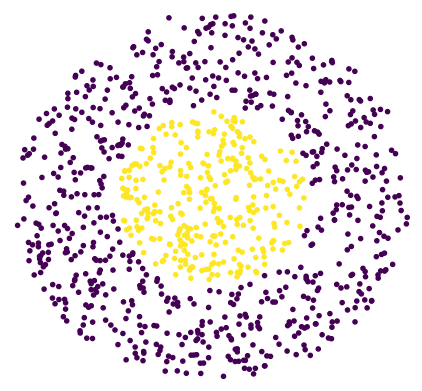}}}
    \vskip 0pt
    \subfloat[\centering D31]{
    \frame{\includegraphics[height=0.2\textwidth,width=0.2\textwidth]{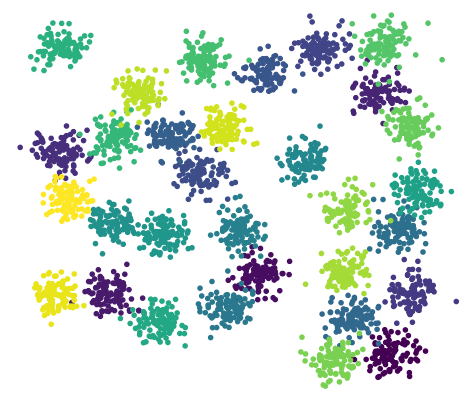}}}
    \quad
    \subfloat[\centering flame]{
    \frame{\includegraphics[height=0.2\textwidth,width=0.2\textwidth]{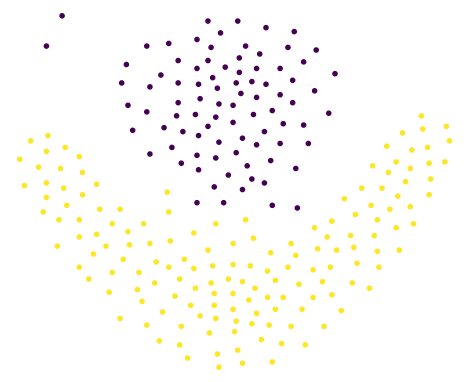}}}
    \quad
    \subfloat[\centering square3]{
    \frame{\includegraphics[height=0.2\textwidth,width=0.2\textwidth]{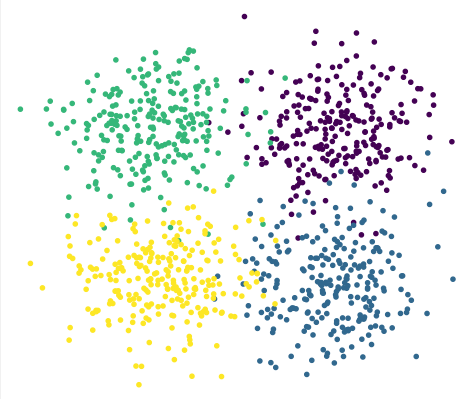}}}
    
    \caption{Examples for rank-differences of synthetic datasets.}
    \label{fig:3}
\end{figure}

Clusters in datasets have great diversity so that the diversity of clustering methods and CVIs is necessary. Since the preferences of CVIs are difficult to analyze precisely and quantitatively, more studies for selecting a proper CVI to measure clusters without true labels should be performed in the future. Having more CVIs expands the options. And before the breakthrough that we discover approaches to select an optimal CVI to measure clusters, it is meaningful to provide more effective CVIs and apply more than one CVI to evaluate clustering results.

\begin{figure}
    \centering
    \subfloat[\centering{Real clusters:\newline DSI $\approx0.456$}]{
    \frame{\includegraphics[height=0.25\textwidth,width=0.25\textwidth]{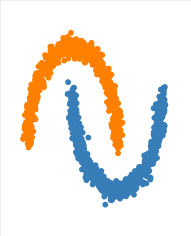}}}
    \qquad
    \subfloat[\centering{Predicted clusters:\newline DSI $\approx0.664$}]{
    \frame{\includegraphics[height=0.25\textwidth,width=0.25\textwidth]{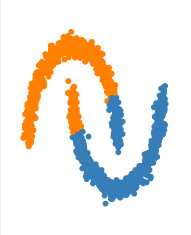}}}
    
    \caption{Wrongly-predicted clusters have a higher DSI score than real clusters.}
    \label{fig:4}
\end{figure}

In addition, to evaluate CVIs is also an important task. Its general process is:
\begin{enumerate}
    \item To create different clusters from datasets.
    \item To compute external CVI with true labels as ground truth and internal CVIs.
    \item To compare results of internal CVIs with the ground truth. Results from an effective internal CVI should be close to the results of an external CVI.
\end{enumerate}

In this paper, we generated different clusters using a variety of clustering methods. To generate different clusters can also be achieved through changing parameters of clustering algorithms (\textit{e.g.}, the number of clusters $k$ in k-means clustering) or taking subsets of datasets. The comparison step could also apply other methods besides the two evaluation metrics: hit-the-best and rank-difference that we have used.

\section{Conclusion}
To evaluate clustering results, it is essential to apply various CVIs because there is no universal CVI for all datasets and no specific method for selecting a proper CVI to measure clusters without true labels. In this paper, we propose the DSI as a novel CVI based on a data separability measure. Since the goal of clustering is to separate a dataset into clusters, we hypothesize that better clustering could cause these clusters to have a higher separability.

Including the proposed DSI, we applied nine internal CVI and one external CVI (ARI) as ground truth to clustering results of five clustering algorithms on various real and synthetic datasets. The results show DSI to be an effective, unique, and competitive CVI to other CVIs compared here. And we summarized the general process to evaluate CVIs and used two methods to compare the results of CVIs with ground truth. We created the rank-difference as an evaluation metric to compare two score sequences. This metric avoids two disadvantages of the hit-the-best measure, which is commonly used in CVI evaluation. We believe that both the DSI and rank-difference metric can be helpful in clustering analysis and CVI studies in the future.



\appendix
\section{Proof of Theorem \ref{thm:1}}
\label{sec:proof}

Consider two classes $X$ and $Y$ that have the same distribution (covering the same region) and have sufficient data points. Suppose $X$ and $Y$ have $N_x$ and $N_y$ data points, and assume the sampling density ratio is $\frac{N_y}{N_x} =\alpha$. Before providing the proof of Theorem \ref{thm:1}, we firstly prove Lemma \ref{lma:1}, which will be used later.

\begin{lemma} \label{lma:1}
If and only if two classes $X$ and $Y$ have the same distribution covering region $\Omega$ and $\frac{N_y}{N_x} =\alpha$, for any sub-region $\Delta \subseteq \Omega$, with $X$ and $Y$ having $n_{xi},n_{yi}$ points, $\frac{n_{yi}}{n_{xi}} =\alpha$ holds.
\end{lemma}

\begin{proof}
Assume the distributions of $X$ and $Y$ are $f(x)$ and $g(y)$. In the union region of $X$ and $Y$, arbitrarily take one tiny cell (region) $\Delta_i$ with $n_{xi}=\Delta_if(x_i)N_x, n_{yi}=\Delta_ig(y_j)N_y; x_i=y_j$. Then,
\[
\frac{n_{yi}}{n_{xi}}=\frac{\Delta_ig(x_i)N_y}{\Delta_if(x_i)N_x}=\alpha \frac{g(x_i)}{f(x_i)}
\]

Therefore:
\[
\alpha \frac{g(x_i)}{f(x_i)} =\alpha \Leftrightarrow \frac{g(x_i)}{f(x_i)}=1 \Leftrightarrow \forall x_i:g(x_i)=f(x_i)
\]
\end{proof}

\subsection{Sufficient condition}
\textbf{Sufficient condition of Theorem \ref{thm:1}.} \textit{If the two classes $X$ and $Y$ with the same distribution and have sufficient data points, then the distributions of the ICD and BCD sets are nearly identical.}

\begin{figure}[h]
    \centerline{\includegraphics{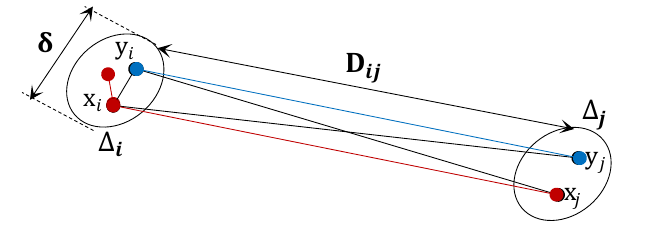}}
    \caption{Two non-overlapping small cells}
    \label{fig:5}
\end{figure}

\begin{proof}
Within the area, select two tiny non-overlapping cells (regions) $\Delta_i$  and $\Delta_j$ (Figure \ref{fig:5}). Since $X$ and $Y$ have the same distribution but in general different densities, the number of points in the two cells $n_{xi},n_{yi};n_{xj},n_{yj}$ fulfills:
\[
\frac{n_{yi}}{n_{xi}} =\frac{n_{yj}}{n_{xj}} =\alpha
\]

The scale of cells is $\delta$, the ICDs and BCDs of $X$ and $Y$ data points in cell $\Delta_i$ are approximately $\delta$ because the cell is sufficiently small. By the Definition~\ref{def:1}~and~\ref{def:2}:
\[
d_{x_i}\approx d_{x_i,y_i}\approx \delta;\quad x_i,y_i\in \Delta_i
\]

Similarly, the ICDs and BCDs of $X$ and $Y$ data points between cells $\Delta_i$  and $\Delta_j$ are approximately the distance between the two cells $D_{ij}$:
\[
d_{x_{ij}}\approx d_{x_i,y_j}\approx d_{y_i,x_j}\approx D_{ij};\; x_i,y_i\in \Delta_i;\, x_j,y_j\in \Delta_j
\]

First, divide the whole distribution region into many non-overlapping cells. Arbitrarily select two cells $\Delta_i$  and $\Delta_j$ to examine the ICD set for $X$ and the BCD set for $X$ and $Y$. By the Definition \ref{def:1} and \ref{def:2}:

\romannum{1} ) The ICD set for $X$ has two distances: $\delta$ and $D_{ij}$, and their numbers are:
\[
d_{x_i}\approx \delta;\; x_i\in \Delta_i:\; |\{d_{x_i}\}|=\frac{1}{2}n_{xi}(n_{xi}-1)
\]
\[
d_{x_{ij}}\approx D_{ij};\; x_i\in \Delta_i;x_j\in \Delta_j:\; |\{d_{x_{ij}}\}|=n_{xi}n_{xj}
\]

\romannum{2} ) The BCD set for $X$ and $Y$ also has two distances: $\delta$ and $D_{ij}$, and their numbers are:
\[
d_{x_i,y_i}\approx \delta;\; x_i,y_i\in \Delta_i:\; |\{d_{x_i ,y_i}\}|=n_{xi} n_{yi}
\]
\[
d_{x_i,y_j}\approx d_{y_i,x_j}\approx D_{ij};\; x_i,y_i\in \Delta_i;x_j,y_j\in \Delta_j:
\]
\[
|\{d_{x_i,y_j}\}|=n_{xi} n_{yj};\; |\{d_{y_i,x_j}\}|=n_{yi} n_{xj}
\]

Therefore, the proportions of the number of distances with a value of $D_{ij}$ in the ICD and BCD sets are:

For ICDs: 
\[
\frac{|\{d_{x_{ij}} \}|}{|\{d_x \}|} =\frac{2n_{xi} n_{xj}}{N_x (N_x-1)}
\]

For BCDs, considering the density ratio:
\[
\frac{|\{d_{x_i,y_j} \}|+|\{d_{y_i,x_j }\}|}{|\{d_{x,y} \}|} =\frac{\alpha n_{xi} n_{xj}+\alpha n_{xi} n_{xj}}{\alpha N_x^2 }=\frac{2n_{xi} n_{xj}}{N_x^2}
\]

The ratio of proportions of the number of distances with a value of $D_{ij}$ in the two sets is:
\[
\frac{N_x (N_x-1)}{N_x^2}=1-\frac{1}{N_x} \to 1 \; \; (N_x\to \infty)
\]

This means that the number of proportions of the number of distances with a value of $D_{ij}$ in the two sets is equal. We then examine the proportions of the number of distances with a value of $\delta$ in the ICD and BCD sets.

For ICDs:
\[
\sum_{i} \frac{|\{d_{x_i}\}|}{|\{d_x\}|} = \frac{\sum_{i} [n_{xi} (n_{xi}-1)]}{N_x (N_x-1)} = \frac{\sum_{i} (n_{xi}^2-n_{xi} )}{N_x^2-N_x} = \frac{\sum_{i} (n_xi^2 ) -N_x}{N_x^2-N_x}
\]

For BCDs, considering the density ratio: 
\[
\sum_{i} \frac{|\{d_{x_i,y_i } \}|}{|\{d_{x,y})\}|} = \frac{\sum_{i} (n_{xi}^2 )}{N_x^2}
\]

The ratio of proportions of the number of distances with a value of $\delta$ in the two sets is:
\[
\frac{\sum_{i} (n_{xi}^2 ) }{N_x^2 }\cdot \frac{N_x^2-N_x}{\sum_{i} (n_{xi}^2 ) -N_x } = \sum_{i} \left(\frac{n_{xi}^2}{N_x^2} \right) \cdot \frac{1-\frac{1}{N_x}}{\sum_{i} \left(\frac{n_{xi}^2}{N_x^2} \right) -\frac{1}{N_x}}\to 1 \; \; (N_x\to \infty)
\]

This means that the number of proportions of the number of distances with a value of $\delta$ in the two sets is equal.

In summary, the fact that the proportion of any distance value ($\delta$ or $D_{ij}$) in the ICD set for $X$ and in the BCD set for $X$ and $Y$ is equal indicates that the distributions of the ICD and BCD sets are identical, and a corresponding proof applies to the ICD set for $Y$.
\end{proof}

\subsection{Necessary condition}
\textbf{Necessary condition of Theorem \ref{thm:1}.} \textit{If the distributions of the ICD and BCD sets with sufficient data points are nearly identical, then the two classes $X$ and $Y$ must have the same distribution.}

\begin{remark}
We prove its \textbf{contrapositive}: if $X$ and $Y$ do not have the same distribution, the distributions of the ICD and BCD sets are not identical. We then apply proof by \textbf{contradiction}: suppose that $X$ and $Y$ do not have the same distribution, but the distributions of the ICD and BCD sets are identical.
\end{remark}

\begin{proof}
Suppose classes $X$ and $Y$ have the data points $N_x, N_y$, which $\frac{N_y}{N_x} =\alpha $. Divide their distribution area into many non-overlapping tiny cells (regions). In the $i$-th cell $\Delta_i$, since distributions of $X$ and $Y$ are different, according to Lemma \ref{lma:1}, the number of points in the cell $n_{xi},n_{yi}$ fulfills:
\[
\frac{n_{yi}}{n_{xi}} = \alpha _i; \; \; \exists \alpha _i \neq \alpha
\]

The scale of cells is $\delta$ and the ICDs and BCDs of the $X$ and $Y$ points in cell $\Delta_i$ are approximately $\delta$ because the cell is sufficiently small.
\[
d_{x_i}\approx d_{y_i}\approx d_{x_i,y_i}\approx \delta; \; \; x_i,y_i\in \Delta_i
\]

In the $i$-th cell $\Delta_i$:

\romannum{1}) The ICD of $X$ is $\delta$, with a proportion of:
\begin{equation}
\label{eq:1}
    \sum_{i} \frac{|\{d_{x_i}\}|}{|\{d_x\}|} = \frac{\sum_{i} [n_{xi} (n_{xi}-1)]}{N_x (N_x-1)} = \frac {\sum_{i} (n_{xi}^2-n_{xi} )}{N_x^2-N_x}=\frac{\sum_{i} (n_{xi}^2 ) -N_x}{N_x^2-N_x}
\end{equation}

\romannum{2}) The ICD of $Y$ is $\delta$, with a proportion of:
\begin{multline} \label{eq:2}
    \sum_{i} \frac{|\{d_{y_i}\}|}{|\{d_y\}|} = \frac{\sum_{i} [n_{yi} (n_{yi}-1)]}{N_y (N_y-1)}=\frac {\sum_{i} (n_{yi}^2-n_{yi} )}{N_y^2-N_y}\\
    =\frac{\sum_{i} (n_{yi}^2 ) -N_y}{N_y^2-N_y}\Bigg\rvert_{\substack{N_y=\alpha N_x \\ n_{yi} = \alpha _i n_{xi}}} = \frac{\sum_{i} (\alpha _i^2 n_{xi}^2 ) -\alpha N_x}{\alpha^2 N_x^2-\alpha N_x}
\end{multline}

\romannum{3}) The BCD of $X$ and $Y$ is $\delta$, with a proportion of:
\begin{equation}
\label{eq:3}
    \sum_{i} \frac{|\{d_{x_i,y_i} \}|}{|\{d_{x,y} \}|}=\frac {\sum_{i} (n_{xi} n_{yi} ) }{N_x N_y}=\frac {\sum_{i} (\alpha _i n_{xi}^2 ) }{\alpha N_x^2}
\end{equation}

For the distributions of the two sets to be identical, the ratio of proportions of the number of distances with a value of $\delta$ in the two sets must be 1, that is $\frac{(\ref{eq:3})}{(\ref{eq:1})}=\frac{(\ref{eq:3})}{(\ref{eq:2})}=1$. Therefore,

\begin{multline} \label{eq:4}
    \frac{(\ref{eq:3})}{(\ref{eq:1})}= \frac {\sum_{i} (\alpha _i n_{xi}^2 ) }{\alpha N_x^2} \cdot \frac{N_x^2-N_x}{\sum_{i} (n_{xi}^2 )-N_x}\\
    = \frac{1}{\alpha N_x^2}\sum_{i} (\alpha _i n_{xi}^2 )\cdot \frac {1-\frac {1}{N_x} }{\frac {1}{N_x^2} \sum_{i}(n_{xi}^2 ) -\frac {1}{N_x}}\Bigg\rvert_{N_x\to \infty} \\
    =\frac {1}{\alpha}\cdot \frac{\sum_{i} (\alpha_i n_{xi}^2 ) }{\sum_{i}(n_{xi}^2 )}=1
\end{multline}

Similarly,
\begin{multline} \label{eq:5}
    \frac{(\ref{eq:3})}{(\ref{eq:2})}= \frac {\sum_{i} (\alpha _i n_{xi}^2 ) }{\alpha N_x^2} \cdot \frac{\alpha^2 N_x^2-\alpha N_x}{\sum_{i} (\alpha _i^2 n_{xi}^2 ) -\alpha N_x}\\
    = \frac{\sum_{i} (\alpha _i n_{xi}^2 )}{N_x^2}\cdot \frac {\alpha-\frac {1}{N_x} }{\frac {1}{N_x^2} \sum_{i} (\alpha _i^2 n_{xi}^2 ) -\frac {\alpha}{N_x}}\Bigg\rvert_{N_x\to \infty} \\
    =\alpha \cdot \frac{\sum_{i} (\alpha_i n_{xi}^2 ) }{\sum_{i} (\alpha _i^2 n_{xi}^2 )}=1
\end{multline}

To eliminate the $\sum_{i} (\alpha _i n_{xi}^2 )$ by considering the Eq.~\ref{eq:4}~and~\ref{eq:5}, we have:
\[
\sum_{i}(n_{xi}^2 )=\frac{\sum_{i} (\alpha _i^2 n_{xi}^2 )}{\alpha^2}
\]

Let $\rho_i=\left(\frac{\alpha_i}{\alpha}\right)^2$, then,
\[
\sum_{i}(n_{xi}^2 )=\sum_{i} (\rho_i n_{xi}^2 )
\]

Since $n_{xi}$ could be any value, to hold the equation requires $\rho_i=1$. Hence:
\[
\forall \rho_i=\left(\frac{\alpha_i}{\alpha}\right)^2=1 \Rightarrow\forall \alpha_i=\alpha
\]

This contradicts $\exists \alpha_i\neq \alpha$. Therefore, the contrapositive proposition has been proved. 
\end{proof}

\section{Synthetic Datasets}
\label{sec:syn_names}

\begin{table}
\tbl{Names of the 97 used synthetic datasets from the Tomas Barton repository$^a$}
{\begin{tabular}{cccccc} \toprule
3-spiral & 2d-10c & ds2c2sc13 & rings & square5 & complex8 \\
aggregation & 2d-20c-no0 & ds3c3sc6 & shapes & st900 & complex9 \\
2d-3c-no123 & threenorm & ds4c2sc8 & simplex & target & compound \\
dense-disk-3000 & triangle1 & 2d-4c & sizes1 & tetra & donutcurves \\
dense-disk-5000 & triangle2 & 2dnormals & sizes2 & curves1 & donut1 \\
elliptical\_10\_2 & dartboard1 & engytime & sizes3 & curves2 & donut2 \\
elly-2d10c13s & dartboard2 & flame & sizes4 & D31 & donut3 \\
2sp2glob & 2d-4c-no4 & fourty & sizes5 & twenty & zelnik1 \\
cure-t0-2000n-2D & 2d-4c-no9 & xor & smile1 & aml28 & zelnik2 \\
cure-t1-2000n-2D & pmf & hepta & smile2 & wingnut & zelnik3 \\
twodiamonds & diamond9 & hypercube & smile3 & xclara & zelnik5 \\
spherical\_4\_3 & disk-1000n & jain & atom & R15 & zelnik6 \\
spherical\_5\_2 & disk-3000n & lsun & blobs & pathbased &   \\
spherical\_6\_2 & disk-4000n & long1 & cassini & square1 &   \\
chainlink & disk-4500n & long2 & spiral & square2 &   \\
spiralsquare & disk-4600n & long3 & circle & square3 &   \\
gaussians1 & disk-5000n & longsquare & cuboids & square4 &  \\
 \botrule
\end{tabular}}
\begin{tabnote}
$^{a.}$ Available at \url{https://github.com/deric/clustering-benchmark/tree/master/src/main/resources/datasets/artificial}.
\end{tabnote}
\label{tab:syn_data}
\end{table}


\bibliographystyle{ws-ijait}
\bibliography{ref}

\end{document}